\documentclass[11pt]{article}
\usepackage[margin=1in]{geometry}
\usepackage[T1]{fontenc}
\usepackage{lmodern,microtype}
\usepackage{amsmath,amssymb,amsthm}
\usepackage[dvipsnames]{xcolor}
\usepackage[colorlinks=true,linkcolor=MidnightBlue,citecolor=MidnightBlue,urlcolor=MidnightBlue]{hyperref}
\DeclareMathOperator*{\E}{\mathbb{E}}
\newtheorem{theorem}{Theorem}
\newtheorem{lemma}{Lemma}
\newtheorem{definition}{Definition}
\title{Alignment via Generalization}
\author{Shai Shalev-Shwartz \and Amnon Shashua}
\date{October 2026}
\begin{document}
\maketitle

\begin{abstract}
AI systems can become better at achieving the goals we give them without becoming better at doing what we actually want. Studying this alignment problem rigorously requires a framework connecting measurable performance to intended behavior. We propose \emph{Aligned-PAC learning}, whose central idea is to reduce alignment to generalization: does success on the requirements used for training extend to requirements the system has not encountered? Rather than express our entire intention through a single training objective, we represent its different aspects through a collection of specific requirements. This perspective already yields strong theoretical results. We show how independent evaluation can provide rigorous evidence of alignment across the represented requirements, and establish how much evaluation is needed. We also develop an optimization method that, under explicit assumptions, combines learned models to satisfy the training requirements. The framework does not solve alignment in full: choosing requirements that faithfully capture our intentions, and understanding when learning will generalize to unseen requirements, remain open challenges. It provides a foundation for making these questions precise and studying them with the tools of learning theory.
\end{abstract}

\section{Introduction}

AI models have become increasingly powerful, with some of their most impressive achievements in mathematical reasoning. These successes suggest a compelling recipe: specify an outcome, reward the model for achieving it, and let learning discover how to get there. Yet the effectiveness of this recipe depends on a crucial assumption: that the objective we optimize faithfully represents the objective we actually care about. When this assumption fails, a model can become better at optimizing its reward without becoming better at the intended task. We define the alignment problem as the gap between the \emph{true objective} and the \emph{proxy} used for learning.

This gap is not new to machine learning. In binary classification, for example, we may care about the probability of a prediction error, measured by the zero-one loss, while optimizing the hinge loss for computational convenience. A gap can also arise from the choice of data. A cow detector trained and evaluated on images of cows in grassy fields may rely on the background and fail to recognize cows standing on sand. Even with the right loss function, success on the selected distribution need not imply success in the situations we care about. Thus, a learning proxy specifies both a loss function and a data distribution. In our framework, we additionally specify an acceptable-loss threshold for each proxy, making explicit when a requirement is satisfied.

Mathematics offers an exceptional case in which this gap can disappear. If the task is to prove a precisely stated theorem and success is determined by a sound proof checker, the criterion used to evaluate the model coincides with the desired outcome: a valid proof. This does not mean that every reward used in mathematical training is perfectly aligned, but it illustrates the advantage of a domain in which correctness can be specified precisely and verified independently. In many other domains, defining an equally faithful criterion is itself a central difficulty.

The alignment problem becomes more challenging as both our ambitions and our models grow. First, we ask models to perform tasks whose objectives are harder to specify and whose failures are harder to diagnose. ``Detect a cow'' admits a relatively direct assessment. In autonomous driving, however, our true objective is ``drive safely and comfortably''. This requires judgments across a much broader range of situations and tradeoffs, and we do not know how to fully express this objective as a loss function. A seemingly natural proxy is to imitate human drivers. Yet human behavior is not always desirable, and better imitation does not necessarily mean safer driving: a model may improve its average imitation loss by predicting routine trajectories more accurately, while responding less safely to rare hazards. Thus, improvement under the proxy can come at the expense of the true objective.

Second, more capable models can exploit subtler gaps in the proxy. A model asked to optimize code may produce a sophisticated implementation that passes every supplied test and achieves excellent benchmark performance, yet contains a subtle bug on an untested input~\cite{doubleai2026maxwell}. Such failures can arise from exploiting weaknesses in the evaluation procedure---a phenomenon known as reward hacking---or simply from failing to generalize beyond the supplied tests. Deliberate deception is not necessary for either failure: optimizing an incomplete criterion can be enough~\cite{williamsking2025safety,bengio2026agents}. Greater capability therefore makes the quality of the proxy more consequential.

Discussions of AI alignment often emphasize the three desiderata of being \emph{honest, harmless, and helpful}. These are important objectives, but naming them does not resolve the underlying difficulty. We still need to translate them into criteria that can guide learning and evaluation, and those criteria may leave gaps. The same difficulty arises in classification, code generation, and autonomous driving, even when the desired behavior is not naturally described by these three terms. Our focus is the more general problem of relating success under an operational proxy to success under the intended objective.

Alignment is difficult in part because we lack a general theory that makes this relationship precise and yields useful guarantees. Our central proposal is to reduce alignment to \emph{generalization}. Instead of attempting to encode the true objective in a single loss and distribution, we consider a family of proxies, each capturing an aspect of the intended behavior through its own loss function, data distribution, and acceptable-loss threshold. The learner observes some of these proxies; the question is whether its success extends to proxies it has not observed. This shifts attention from fitting a fixed proxy to generalizing across the requirements that the proxy is meant to represent. Generalization provides both a well-developed theoretical language and a practical evaluation principle: assess performance on held-out proxies, as well as on held-out examples within a proxy.

We formalize this perspective through \emph{Aligned-PAC learning}, a generalization of the Probably Approximately Correct learning framework. Ordinary PAC learning studies performance under a fixed loss and distribution using sampled examples. Aligned-PAC extends this view to a distribution over proxies, each specifying a loss function, an example distribution, and an acceptable-loss threshold, with a rule for combining performance across proxies. This introduces a question of \emph{proxy complexity}: how many proxies must a learner observe to perform well on the broader family? The resulting guarantees necessarily depend on how that family represents the intended objective and on the assumptions connecting observed and unobserved proxies. Aligned-PAC does not eliminate the need to specify meaningful requirements; it provides a framework for studying when learning from a limited set of such requirements can generalize to the others. In this way, we aim to bring the tools of learning theory to the alignment problem.

\section{The Aligned-PAC framework}

PAC learning is a fundamental framework for learning from data. 
In its general form, one needs to define a hypothesis class $\mathcal{H}$, 
an example space $\mathcal{X}$, 
a loss function $\ell : \mathcal{H} \times \mathcal{X} \to \mathbb{R}$, 
and a distribution $\mathcal{D}$ over $\mathcal{X}$. Then, the goal is to learn 
a hypothesis $h \in \mathcal{H}$ that (probably, approximately) minimizes 
the expected loss on the data, $\E_{x \sim \mathcal{D}} [\ell(h, x)]$, while only 
accessing $\mathcal{D}$ via i.i.d. samples.

Deep Learning made PAC learning powerful: the hypothesis class $\mathcal{H}$ is a family of neural networks 
or policies implemented by neural-network-based AI agents, and the learning algorithm searches for a hypothesis 
by minimizing the loss on the data via Stochastic Gradient Descent (SGD) with respect to the parameters of the model. 
This captures pre-training, in which the loss is the negative log-likelihood of the data under the model.
It captures supervised fine-tuning, in which the loss is the cross-entropy loss between 
the model's predictions and the true labels. It also captures reinforcement learning for reasoning: 
$\mathcal{D}$ is a distribution over prompts or tasks, and $\ell(h,x)$ is the expected negative outcome reward 
over trajectories generated by the model $h$ on task $x$. The model searches for intermediate steps 
to reach a final correct solution.

In this formulation, the loss function $\ell$ and the distribution $\mathcal{D}$ specify the objective against which learning is evaluated. A PAC guarantee relates performance on sampled examples to expected performance under this same loss and distribution. It therefore addresses generalization from observed to unseen examples, but does not by itself establish that minimizing this expected loss achieves the true objective of the task.

The alignment problem concerns this latter gap. Even a hypothesis that minimizes $\E_{x \sim \mathcal{D}}[\ell(h,x)]$ over $\mathcal{H}$ may fail to satisfy an important requirement if the loss does not adequately penalize its violation or the distribution assigns too little weight to the situations in which it matters. Such a failure can persist even when there is no estimation or optimization error. To address it, we need to evaluate behavior across the different requirements of the task, allowing each requirement to have its own loss function and data distribution.

We view the alignment problem as a generalization problem - instead of defining a single proxy that is perfectly aligned with the true objective of the task, we define a family of loss functions and data distributions that reflect different properties of the true objective of the task. This culminates in a new framework for alignment, which we call \emph{Aligned-PAC learning}.

To motivate Aligned-PAC, consider the problem of learning an end-to-end self-driving neural network. The goal is to learn a policy that takes as input sensor data (video streams from the cameras, ego car kinematics, etc.) and outputs the future trajectory the car should take. 
Our true objective is to have a safe and comfort driving, but this is hard to formally define. As a proxy loss function we use imitation learning: we aim to learn a distribution over future trajectories that maximizes the likelihood of observed human-driven trajectories.  

While imitating humans sounds like a good proxy, it has several disadvantages. Firstly, some common human behaviors are not desirable for a self-driving car. For example, many humans do not fully stop at a stop sign but instead roll slowly and then continue, while the law dictates full stop. Secondly, in a wide road, there is a high uncertainty on how exactly a human would drive (offsetting a bit to the left or right is totally fine). The model therefore must spread the distribution over different trajectories and suffer loss on the likelihood of the actual trajectory in each instance. Not spreading this distribution perfectly would lead to a slightly excess loss over the optimal distribution. On the other hand, given a rare hazard, predicting a 50\% probability to hit it and 50\% probability to overtake it would yield a much smaller penalty to the total loss, due to the rarity of the event. Now consider two models A and B, where A is better than B in wide roads, but worse in rare hazards. To our proxy loss, model A may appear to be better than model B, while model A is catastrophically worse than B with respect to the true objective. 

Practitioners circumvent this problem by manually designing a family of proxies. Each proxy captures a different aspect of the true objective of the task. For example, one proxy may capture the ability to stop at a stop sign, while another proxy may capture the ability to drive in a wide road. Each proxy relies on an underlying abstraction that captures one aspect of the true objective of the task. By combining these proxies, the practitioner can design a policy that is robust to the different aspects of the true objective of the task.

This motivates an objective of Aligned-PAC learning, which we define as follows. Let $\mathcal{T}$ be a distribution over proxies, where each proxy is a triple $(\ell, \mathcal{D}, \tau)$ consisting of a loss function $\ell : \mathcal{H} \times \mathcal{X} \to \mathbb{R}$, a distribution $\mathcal{D}$ over $\mathcal{X}$, and an acceptable-loss threshold $\tau \in \mathbb{R}$. Different proxies represent different requirements of the intended task. The goal is to find a hypothesis $h \in \mathcal{H}$ that minimizes 
\begin{equation} \label{eq:aligned-pac}
\mathbb{P}_{(\ell, \mathcal{D}, \tau) \sim \mathcal{T}} \left[\E_{x \sim \mathcal{D}} [\ell(h, x)] > \tau\right] ~,
\end{equation}
This objective measures the probability that the hypothesis violates a requirement sampled from $\mathcal{T}$. Once a requirement is satisfied, further reducing its loss yields no additional benefit, preventing improvements on already satisfied requirements from compensating for failures on others.

In the driving example, one proxy can measure imitation loss under ordinary driving conditions, while another measures collision probability under a distribution concentrated on hazardous situations. Each proxy has its own acceptable threshold. Once the imitation requirement is satisfied, more accurate prediction of routine human trajectories cannot compensate for failing the hazard-avoidance requirement. Moreover, the hazard proxy can receive substantial probability mass under the proxy distribution, regardless of how rarely such situations occur in ordinary driving.

More generally, for a function $L : \mathbb{R}^2 \to \mathbb{R}$, we can define the Aligned-PAC objective as
\begin{equation} \label{eq:aligned-pac-general}
\E_{(\ell, \mathcal{D}, \tau) \sim \mathcal{T}} L\left(\E_{x \sim \mathcal{D}} [\ell(h, x)], \tau\right) ~.
\end{equation}
This encompasses Equation \ref{eq:aligned-pac} by setting $L(\epsilon, \tau) = 1[\epsilon > \tau]$.

Equation \ref{eq:aligned-pac-general} is a generalization of the PAC learning objective. Indeed, ordinary PAC learning is recovered when $\mathcal{T}$ is a point-mass distribution on a single proxy and $L(\epsilon, \tau) = \epsilon$. The generalization of PAC to a distribution over proxies allows us to capture the different aspects of the true objective of the task. It allows us to tackle the alignment problem through the lens of \emph{generalization}. 

Aligned-PAC offers two advantages for alignment. First, in practice, it is much easier to express our true intention through a large collection of specific requirements than to encode it in a single proxy. Each proxy can focus on one aspect of the desired behavior and the situations in which it matters, without having to capture the entire objective. Second, generalization across proxies provides a defense against imperfections in the training proxy, whether these are actively exploited through reward hacking or simply reflect an incomplete or mistaken specification. Success on independently sampled, unseen proxies requires behavior that extends beyond the particular criteria used for training. A train/validation split at the level of proxies therefore assesses whether the learner has captured requirements that the training proxy left unspecified. Good validation performance indicates that the learned behavior generalizes across the aspects of our intention represented by the proxy family.

To specify how the learner accesses the objective, recall that in PAC learning it receives $m$ i.i.d. examples from $\mathcal{D}$. The sample complexity describes how large $m$ must be to achieve a desired accuracy $\epsilon$ with confidence $1-\delta$, as a function of the relevant complexity measures of the hypothesis class and the learning algorithm.

In Aligned-PAC, sampling occurs at two levels. We first draw $n$ proxies $(\ell_i, \mathcal{D}_i, \tau_i)$ independently from $\mathcal{T}$. The learner receives the loss function $\ell_i$ and threshold $\tau_i$ for each proxy, together with sampling access to $\mathcal{D}_i$. At each step, it may choose a proxy based on its observations so far and obtain a fresh example from that proxy's distribution. Conditional on the sampled proxies, the sampling history, and the selected proxy index $i$, the next example is a fresh draw from $\mathcal{D}_i$. The learner may allocate at most $m$ examples in total across the proxies. Writing $m_i$ for the number of examples drawn from $\mathcal{D}_i$, we require $\sum_{i=1}^n m_i \leq m$; the allocation may be adaptive. Based on these observations, the learner outputs a hypothesis $h \in \mathcal{H}$ (or an ensemble of such hypotheses). Equal allocation is a special case suitable for the initial analysis: the learner draws $\lfloor m/n \rfloor$ examples from each proxy's example distribution.

This protocol gives rise to two notions of complexity: the \emph{sample complexity}, describing the required total example budget $m$, and the \emph{proxy complexity}, describing the required number $n$ of sampled proxies. We can study how both quantities must grow with the accuracy $\epsilon$, the confidence parameter $\delta$, and the relevant complexity measures of the hypothesis class and the learning algorithm, including its sampling strategy. These requirements may depend on one another, so the learning guarantee is naturally expressed in terms of pairs $(n,m)$. Increasing $m$ allows the learner to gather more information about the observed requirements, while increasing $n$ exposes it to more requirements from the proxy distribution.

\section{Validating Alignment}
\label{sec:validation}

Suppose we already have a hypothesis $h$, trained on a known collection of proxies. We want to determine whether its success on those requirements extends to the population of requirements. We first use the smooth loss
\[
L(u,\tau)=[u-\tau]_+^2,\qquad [a]_+=\max\{a,0\}.
\]
This loss is zero when the threshold criterion is satisfied and penalizes violations quadratically. Its derivative with respect to $u$ is $2$-Lipschitz, which makes it possible to control the effect of estimating each proxy's expected loss. We assume that all inner losses and thresholds lie in $[0,1]$.

For a proxy $t=(\ell_t,\mathcal{D}_t,\tau_t)$ and a distribution $\mathcal Q$ over proxies, write
\[
r_t(h)=\E_{x\sim\mathcal D_t}[\ell_t(h,x)],
\qquad
R(h;\mathcal Q)=\E_{t\sim\mathcal Q}[r_t(h)-\tau_t]_+^2.
\]
Let $\mathcal T_{tr}$ be the uniform distribution over the sequence of training proxies, counting repetitions.
\begin{definition}[Mis-alignment gap]
The mis-alignment gap of $h$ is
\[
G(h)=R(h;\mathcal T)-R(h;\mathcal T_{tr}).
\]
\end{definition}
A large positive gap means that success on the training requirements overstates performance on the population of requirements. A small gap together with small training-proxy risk certifies small population risk; a small gap alone does not.

We fix $h$ before evaluation. We may draw fresh examples from the example distribution of any training proxy, and we may independently sample fresh proxies from $\mathcal T$ and examples from their distributions. To sample from $\mathcal T_{tr}$, we simply select a training-proxy index uniformly with replacement; even when an index is selected repeatedly, its evaluation examples are fresh. Thus, estimating the average training-proxy risk does not require evaluating every training proxy individually.

We ask three questions: how can we estimate $G(h)$, how many fresh examples are needed, and how many fresh proxies are needed? Specifically, we seek an estimate $\widehat G$ satisfying
\begin{equation}\label{eq:gap-accuracy}
\mathbb P\bigl(|\widehat G-G(h)|\leq\epsilon\bigr)\geq1-\delta.
\end{equation}
The \emph{sample complexity} counts all fresh examples, including examples from training proxies. The \emph{proxy complexity} counts fresh draws from $\mathcal T$; selecting an existing training-proxy index does not count as a fresh proxy. All guarantees below hold conditional on the training process, for any returned $h$, and require no complexity bound on the hypothesis class.

\subsection{A simple estimator with equal allocation}

We first estimate each risk separately. For $\mathcal Q\in\{\mathcal T,\mathcal T_{tr}\}$, sample $N$ proxies independently from $\mathcal Q$, draw $k$ fresh examples per proxy, and compute
\[
\widehat R_{N,k}(h;\mathcal Q)
=\frac1N\sum_{i=1}^N
\left[\frac1k\sum_{j=1}^k\ell_{t_i}(h,x_{i,j})-\tau_{t_i}\right]_+^2.
\]
Subtract the estimate for $\mathcal T_{tr}$ from the estimate for $\mathcal T$. This construction uses $N$ fresh population proxies and $2Nk$ fresh examples in total.

\begin{theorem}[Equal-allocation validation]\label{thm:equal-validation}
For $0<\epsilon,\delta<1$, take
\[
N=\left\lceil\frac{8\log(4/\delta)}{\epsilon^2}\right\rceil,
\qquad k=\left\lceil\frac1\epsilon\right\rceil.
\]
The estimator above satisfies Equation~\ref{eq:gap-accuracy}. Its sample complexity is $O(\epsilon^{-3}\log(1/\delta))$ and its proxy complexity is $O(\epsilon^{-2}\log(1/\delta))$ for $\delta\leq1/4$.
\end{theorem}

The proof, given in Appendix~\ref{app:equal-validation}, separates variation across sampled proxies from bias within each proxy. Averaging $N$ bounded proxy scores contributes error of order $N^{-1/2}$. Smoothness bounds the bias from using $k$ examples by $1/(4k)$. Hence, order $\epsilon^{-2}$ proxies and $\epsilon^{-1}$ examples per proxy suffice, giving the cubic total budget.

This bias cannot be removed merely by sharpening concentration. If a proxy has Bernoulli losses with mean $1/2$ and threshold $1/2$, its true squared violation is zero, but its empirical squared violation has expectation exactly $1/(8k)$. Averaging more such scores concentrates around a biased value. These biases need not cancel between the population and training-proxy distributions. A better sample bound therefore calls for a different estimator.

\subsection{A multilevel estimator with quadratic sample complexity}

The equal-allocation estimator assesses every sampled proxy at the same resolution. Instead, we can estimate a coarse risk using many inexpensive assessments and then estimate corrections using fewer, more expensive assessments. We use an antithetic multilevel Monte Carlo construction for nested expectations~\cite{giles2018nested}. The sample sizes and the number of proxies at each level are fixed in advance: the method uses unequal allocation, but does not require adaptive sampling.

We describe one estimate of $R(h;\mathcal Q)$, where $\mathcal Q$ is either the population proxy distribution $\mathcal T$ or the uniform distribution $\mathcal T_{tr}$ over training-proxy indices. Throughout, $h$ is fixed and $L(u,\tau)=[u-\tau]_+^2$. A level specifies the number of examples used to assess a proxy: level $j$ uses $2^j$ examples.

To generate one correction at level $j\geq1$, sample $t\sim\mathcal Q$ and draw $2^j$ fresh examples independently from $\mathcal D_t$. Let $A_j$ and $B_j$ be the average inner losses in the first and second halves of this sample, respectively. Thus, each half contains $2^{j-1}$ examples, and $(A_j+B_j)/2$ is the average over all $2^j$ examples. Define the correction
\begin{equation}\label{eq:multilevel-correction}
\Delta_j
=L\left(\frac{A_j+B_j}{2},\tau_t\right)
-\frac{L(A_j,\tau_t)+L(B_j,\tau_t)}2.
\end{equation}
The expectation of $\Delta_j$ is the difference between the expected empirical scores using $2^j$ and $2^{j-1}$ examples per proxy. At level zero, use one example and set $\Delta_0=L(\ell_t(h,x),\tau_t)$. Consequently, the expectations telescope: summing them through level $J$ gives the expected empirical score using $2^J$ examples per proxy.

Choose a largest level $J$ and a number $N_j$ of correction replicates at each level. For every $j\in\{0,\ldots,J\}$ and $a\in\{1,\ldots,N_j\}$, let $\Delta_j^{(a)}$ denote an independent copy of the level-$j$ correction just defined: draw a new proxy from $\mathcal Q$ and new examples for that copy. All copies are independent across both $a$ and $j$, conditional on training; within a copy, the full-sample and half-sample scores deliberately share the same examples. Average the copies at each level and add the resulting estimates:
\[
Z_{\mathcal Q}=\sum_{j=0}^J\frac1{N_j}
\sum_{a=1}^{N_j}\Delta_j^{(a)},
\]
The multilevel estimator is generally biased as well. Since its expectation is the expected empirical score based on $2^J$ examples per proxy, a single run satisfies
\[
0\leq\E Z_{\mathcal Q}-R(h;\mathcal Q)\leq\frac1{4\cdot2^J}.
\]
This is the same bias bound as for the equal-allocation risk estimator with $k=2^J$. The advantage is that multilevel estimation controls this bias with fewer total examples: only the fine-level corrections use $2^J$ examples per proxy, while most assessments use much smaller samples. Taking medians provides confidence amplification, not unbiasedness; the final guarantee controls total estimation error.

Thus, one run uses $\sum_{j=0}^J N_j$ proxy draws and $\sum_{j=0}^J N_j2^j$ examples. Choose $2^J$ of order $\epsilon^{-1}$ and $N_j$ proportional to $\epsilon^{-2}2^{-3j/2}$, rounded up. To obtain high confidence, repeat the entire construction independently and take the median of the resulting values of $Z_{\mathcal Q}$. Apply this procedure separately to $\mathcal T$ and $\mathcal T_{tr}$, with accuracy $\epsilon/2$ and failure probability $\delta/2$ for each risk, and subtract the two medians. Appendix~\ref{app:multilevel-validation} gives the exact allocations and number of repetitions.

The reason this saves samples is cancellation. The two half-sample means average to the full-sample mean, so their first-order errors cancel in Equation~\ref{eq:multilevel-correction}. Smoothness makes the correction quadratic in $A_j-B_j$. Its variance is therefore $O(2^{-2j})$, while its sample cost is only $2^j$. The fine corrections can be estimated with progressively fewer proxies, and the resulting total sample cost is a convergent geometric sum times $\epsilon^{-2}$.

\begin{theorem}[Multilevel validation]\label{thm:multilevel-validation}
There is a nonadaptive estimator satisfying Equation~\ref{eq:gap-accuracy} with both sample complexity and proxy complexity
\[
O\!\left(\epsilon^{-2}\log(1/\delta)\right)
\]
for $0<\epsilon\leq1/2$ and $0<\delta\leq1/4$. These are deterministic budget bounds, with universal implied constants. The budgets include estimating both risks; only draws from $\mathcal T$ count toward fresh-proxy complexity.
\end{theorem}

Appendix~\ref{app:multilevel-validation} specifies the estimator and proves the theorem. The improvement is in the total example budget: the proxy complexity retains the quadratic rate. Although individual corrections can be negative, their telescoping sum estimates the same nonnegative squared-violation risk. The construction therefore changes the estimator, not the objective.

\subsection{Matching lower bounds}

The quadratic dependence is unavoidable in general. Even with a single hypothesis and two proxy types, estimating the gap can require distinguishing two Bernoulli means separated by order $\epsilon$. Conversely, even if the score of each sampled proxy is revealed exactly, estimating its population average can require order $\epsilon^{-2}$ independent proxies.

\begin{theorem}[Validation lower bounds]\label{thm:validation-lower}
There are universal constants $c,\epsilon_0>0$ such that, for $0<\epsilon\leq\epsilon_0$ and $0<\delta\leq1/4$, any estimator guaranteeing Equation~\ref{eq:gap-accuracy} uniformly over the bounded-loss setting must, in the worst case, use at least
\[
c\epsilon^{-2}\log(1/\delta)
\]
fresh examples. There is also a family of instances requiring at least this order of fresh population proxies, even if their exact scores are supplied without example sampling. The lower bounds allow adaptive evaluation and apply conditional on a training outcome providing no information about the unknown parameter.
\end{theorem}

Appendix~\ref{app:validation-lower} proves the two claims by reduction to binary hypothesis testing. Thus, under our sampling-access model, the multilevel estimator has optimal worst-case sample and proxy complexity, up to universal constants. These guarantees concern the fixed hypothesis under the specified proxy distribution; using validation to modify or select the hypothesis requires further independent evaluation or an analysis accounting for that selection.

\subsection{Margin-based validation}
\label{sec:margin-validation}

A complementary goal is to certify that only a small probability mass of requirements is violated by more than a prescribed margin. For this purpose, return to the pass/fail loss and its margin version:
\[
L(u,\tau)=\mathbf{1}[u>\tau],
\qquad L_\gamma(u,\tau)=\mathbf{1}[u>\tau+\gamma],
\]
where $\gamma\in(0,1)$. Define
\[
R_\gamma(h;\mathcal Q)
=\E_{t\sim\mathcal Q}L_\gamma(r_t(h),\tau_t)
=\mathbb P_{t\sim\mathcal Q}\bigl(r_t(h)>\tau_t+\gamma\bigr).
\]
This is a probability of margin violations, rather than the squared-violation risk studied above.

Fix $h$ independently of evaluation. Draw $n_v$ proxies independently from $\mathcal Q$ and $k_v$ fresh examples from each proxy's example distribution, independently conditional on the proxies. Accept $h$ only if every empirical proxy passes:
\[
\widehat r_{t_i}(h)
=\frac1{k_v}\sum_{j=1}^{k_v}\ell_{t_i}(h,x_{i,j})
\leq\tau_{t_i},\qquad i=1,\ldots,n_v.
\]
The intuition is that a proxy whose true loss exceeds its threshold by more than $\gamma$ is detected with constant probability using only $O(\gamma^{-2})$ examples. If more than an $\epsilon$ fraction of requirements have such violations, each independent evaluation has probability at least a constant times $\epsilon$ of detecting one. The probability that all evaluations pass therefore decays exponentially in $n_v\epsilon$.

\begin{theorem}[Margin-based validation]\label{thm:margin-validation}
Assume inner losses lie in $[0,1]$. For any fixed $h$ with $R_\gamma(h;\mathcal Q)>\epsilon$, the probability that all $n_v$ empirical proxies pass is at most
\[
\exp\!\left(-n_v\epsilon\bigl(1-e^{-2k_v\gamma^2}\bigr)\right).
\]
In particular, for $\epsilon,\gamma,\delta\in(0,1)$, take
\[
k_v=\left\lceil\frac{\log 2}{2\gamma^2}\right\rceil,
\qquad
n_v=\left\lceil\frac{2\log(1/\delta)}{\epsilon}\right\rceil.
\]
With probability at least $1-\delta$, either an empirical proxy fails or $R_\gamma(h;\mathcal Q)\leq\epsilon$. For $\delta\leq1/4$, the procedure uses
\[
n_v=O\!\left(\frac{\log(1/\delta)}{\epsilon}\right)
\quad\text{proxy draws},\qquad
m_v=n_vk_v=O\!\left(\frac{\log(1/\delta)}{\epsilon\gamma^2}\right)
\quad\text{fresh examples}.
\]
\end{theorem}

The proof is in Appendix~\ref{app:margin-validation}. It does not require every individual proxy to be evaluated correctly: constant probability of detecting each margin violation suffices, avoiding a union bound over inner estimation errors. Taking $\mathcal Q=\mathcal T$ gives a population certificate and the stated fresh-proxy complexity. Taking $\mathcal Q=\mathcal T_{tr}$ uses only existing training-proxy indices and fresh examples.

The certificate controls false acceptance: a hypothesis with $R_\gamma(h;\mathcal Q)>\epsilon$ passes all evaluations with probability at most $\delta$. It does not assert a posterior probability conditional on acceptance, nor does it guarantee acceptance of every hypothesis with small margin risk. In particular, a proxy whose true loss equals its threshold can fail empirically because of sampling noise. The linear dependence on $1/\epsilon$ is consistent with the preceding lower bounds: this procedure certifies a small margin-violation probability when all proxies pass, rather than estimating an arbitrary mis-alignment gap to additive accuracy $\epsilon$.

\section{Optimizing the Aligned-PAC Objective via Proxy-Boosting}
\label{sec:proxy-boosting}

We now turn to optimizing the Aligned-PAC objective
\[
\E_{(\ell,\mathcal D,\tau)\sim\mathcal T}
L\left(\E_{x\sim\mathcal D}[\ell(h,x)],\tau\right).
\]
Ignoring $L$ and minimizing the average inner loss does not generally solve this problem. It allows improvements on already satisfied requirements to compensate for violations of other requirements, even when the thresholds specify that the former improvements have no value. Direct stochastic gradient descent on the stated objective also presents difficulties. For the pass/fail loss, the outer function is discontinuous and its derivative is zero away from the threshold. Even for a smooth outer loss, inserting a minibatch estimate into the nonlinear function generally yields a biased gradient estimator. These difficulties motivate an approach that repeatedly solves ordinary weighted learning problems while changing the emphasis placed on the proxies.

The rationale is to use boosting as an outer loop around standard deep learning optimization. The outer loop adjusts the weights of the proxies to emphasize requirements that remain poorly satisfied. For each fixed weighting, the inner loop trains a model on an ordinary expected-loss objective: sample a proxy according to its weight, sample an example from that proxy, and minimize the average loss over the resulting examples using SGD. Thus, each inner problem has the familiar form of minimizing an average of individual losses, while the outer loop handles the requirement thresholds and combines the learned models into an ensemble. Our oracle assumption formalizes the practical premise that motivates much of deep learning: an optimizer such as SGD can find a sufficiently good model for such an objective. The assumption concerns population performance for every weighting encountered, so it is stronger than merely reducing a minibatch or training loss. Subject to this assumption, boosting converts progress on these ordinary learning problems into simultaneous satisfaction of the training requirements.

\subsection{Margin realizability and a population learning oracle}

We focus on $L(u,\tau)=\mathbf{1}[u>\tau]$ and assume that inner losses and thresholds lie in $[0,1]$. Suppose there exists a fixed $h^*\in\mathcal H$ such that
\begin{equation}\label{eq:proxy-realizability}
\mathbb P_{t\sim\mathcal T}\bigl(r_t(h^*)\leq\tau_t-\gamma\bigr)=1
\end{equation}
for some $\gamma\in(0,1)$. This is an analogue of separability with margin: one hypothesis satisfies almost every requirement with a common positive slack.

Draw $n$ training proxies $t_i=(\ell_i,\mathcal D_i,\tau_i)$ independently from $\mathcal T$, and write $r_i(h)=\E_{x\sim\mathcal D_i}\ell_i(h,x)$. With probability one, the same $h^*$ satisfies all sampled proxies with margin $\gamma$. Let $\Delta_n=\{q\in[0,1]^n:\sum_iq_i=1\}$ be the simplex of distributions over these proxies. Then
\[
\begin{aligned}
\frac1n\sum_{i=1}^n\mathbf{1}[r_i(h^*)>\tau_i-\gamma]
&=\max_i\mathbf{1}[r_i(h^*)>\tau_i-\gamma]\\
&=\max_{q\in\Delta_n}\sum_{i=1}^n q_i
\mathbf{1}[r_i(h^*)>\tau_i-\gamma]=0.
\end{aligned}
\]
Moreover, the signed violations satisfy
\[
\max_{q\in\Delta_n}\sum_{i=1}^nq_i\bigl(r_i(h^*)-\tau_i\bigr)
=\max_i\bigl(r_i(h^*)-\tau_i\bigr)\leq-\gamma.
\]
Thus, the same hypothesis achieves a negative margin under every weighting of the training proxies. This suggests an iterative procedure that emphasizes poorly satisfied proxies while asking the learner to achieve a negative weighted violation. In this section, the learning oracle acts on their true expected losses; there is no preliminary step of replacing them by a fixed empirical sample.

For any distribution $q$ over the training proxies, assume that the oracle returns $h_q\in\mathcal H$ satisfying
\begin{equation}\label{eq:weighted-learning-oracle}
\sum_{i=1}^n q_i r_i(h_q)
\leq\inf_{h\in\mathcal H}\sum_{i=1}^n q_i r_i(h)+\xi,
\qquad 0\leq\xi<\gamma.
\end{equation}
Realizability immediately implies a weighted improvement over the thresholds:
\begin{equation}\label{eq:proxy-weighted-edge}
\sum_{i=1}^nq_i\bigl(r_i(h_q)-\tau_i\bigr)
\leq-\rho,\qquad \rho=\gamma-\xi>0.
\end{equation}
The analysis only needs Equation~\ref{eq:proxy-weighted-edge}; approximate weighted minimization is a sufficient condition. We require the guarantee for each oracle call, including calls whose weights depend on earlier samples. This is an assumption about the learning procedure, not a consequence of realizability or a claim that SGD always finds a global optimum.

The weighted population objective can be accessed by sampling proxy $i$ with probability $q_i$, drawing $x\sim\mathcal D_i$, and using the loss $\ell_i(h,x)$. When differentiation and expectation can be interchanged, this also gives an unbiased stochastic gradient. Thus, the oracle need not explicitly calculate population risks. Its sample and computational requirements are separate from the samples used below to update the proxy weights.

\subsection{Proxy-boosting with sampled weight updates}

Although the oracle guarantee concerns expected losses, the algorithm updates its weights using only fresh observations. Initialize $q_i^{(1)}=1/n$ and fix a number of rounds $B$, a learning rate $\eta>0$, and a batch size $k\geq1$. For $s=1,\ldots,B$:
\begin{enumerate}
\item Apply the population learning oracle with weights $q^{(s)}$ to obtain $h_s$.
\item After $h_s$ is chosen, draw $k$ fresh examples $x_{i,1}^{(s)},\ldots,x_{i,k}^{(s)}$ from each $\mathcal D_i$ and compute
\[
\widehat r_i^{(s)}=\frac1k\sum_{j=1}^k\ell_i(h_s,x_{i,j}^{(s)}),
\qquad \widehat v_i^{(s)}=\widehat r_i^{(s)}-\tau_i.
\]
These examples are independent across proxies and within each batch conditional on the preceding history and $h_s$. They are fresh relative to that history, including any samples used to select $h_s$.
\item Update the proxy distribution by
\[
q_i^{(s+1)}=
\frac{q_i^{(s)}\exp(\eta\widehat v_i^{(s)})}
{\sum_{a=1}^n q_a^{(s)}\exp(\eta\widehat v_a^{(s)})}.
\]
\end{enumerate}
The update emphasizes proxies with larger observed signed violations. It follows the boosting principle of emphasizing difficult requirements, using multiplicative weights over proxies~\cite{arora2012multiplicative}. The signed violations retain the slack provided by the margin assumption; the update does not require accurate pass/fail classification of every proxy in every round.

Return the randomized ensemble $\pi$ that selects $h_s$ uniformly from $h_1,\ldots,h_B$. Writing $\ell_i(\pi,x)$ for its expected loss over this random choice gives
\[
\ell_i(\pi,x)=\E_{s\sim U([B])}[\ell_i(h_s,x)]
=\frac1B\sum_{s=1}^B\ell_i(h_s,x).
\]
Consequently, by linearity of expectation,
\[
r_i(\pi)=\E_{x\sim\mathcal D_i}\E_{s\sim U([B])}[\ell_i(h_s,x)]
=\frac1B\sum_{s=1}^B r_i(h_s).
\]
Operationally, the ensemble selects a constituent uniformly for an evaluation or episode. Its guarantee concerns expected loss over this internal randomization, not the performance of every constituent. 

\begin{theorem}[Population guarantee with sampled updates]\label{thm:proxy-boosting}
Assume Equation~\ref{eq:proxy-realizability} and that the oracle satisfies Equation~\ref{eq:weighted-learning-oracle} at every round. For any fixed $B,k\geq1$, $\eta>0$, and $\delta\in(0,1)$, the sampled-update algorithm satisfies, with probability at least $1-\delta$,
\begin{equation}\label{eq:sampled-proxy-bound}
\max_{i\in[n]}\bigl(r_i(\pi)-\tau_i\bigr)
\leq-\rho+\frac{\log n}{\eta B}+\frac\eta2
+2\sqrt{\frac{\log((n+1)/\delta)}{2Bk}},
\qquad \rho=\gamma-\xi.
\end{equation}
In particular, $k=1$, $\eta=\rho/2$, and
\[
B\geq\frac{32\log((n+1)/\delta)}{\rho^2}
\]
suffice to ensure $r_i(\pi)\leq\tau_i-\rho/4$ for every training proxy with probability at least $1-\delta$. The theorem holds conditional on any training-proxy sequence satisfying the realizability condition, and hence also over the random training proxies and subsequent sampling.
\end{theorem}

The proof in Appendix~\ref{app:proxy-boosting} combines an exponential-potential argument with concentration of accumulated sampling errors. The inequality $\mathbf{1}[z\geq0]\leq e^z$ connects the exponential potential to threshold violations. A proxy with persistently positive cumulative violation would acquire a large exponential weight. The oracle's negative weighted population violation prevents this, up to sampling error. Freshness makes the observed losses conditionally unbiased even though the weights and hypotheses are adaptive. It is therefore enough to control accumulated error over rounds, rather than demanding an accurate estimate in each round. The two sampling-error terms account for individual proxies and for the changing weighted average used by the oracle.

With $k=1$, the algorithm uses $O(\rho^{-2}\log((n+1)/\delta))$ oracle calls and $nB$ fresh examples for weight updates. These counts exclude samples used internally by the oracle; the theorem does not specify the total sample complexity of implementing that oracle. No additional proxies are drawn during boosting.

The guarantee is zero pass/fail loss on the true expected losses of all training proxies, rather than on a finite empirical sample. Generalization to unseen requirements remains a separate question, addressed by independent validation or an additional proxy-complexity analysis.

\section{Related Work}
\label{sec:related-work}

\paragraph{Alignment, reward hacking, and generalization.}
The mismatch between intended behavior and optimized objectives is a central concern in AI safety, including reward hacking and failures under distribution shift~\cite{amodei2016concrete}.
Formal studies characterize reward hacking through the relationship between proxy and true rewards~\cite{skalse2022reward}, including conditions under which their correlation under a reference policy supports guarantees for policies regularized toward it~\cite{laidlaw2025correlated}.
Goal misgeneralization studies already connect alignment failures to generalization: an agent may retain its capabilities in new situations while pursuing an unintended goal~\cite{langosco2022goal,shah2022goal}.
Aligned-PAC makes a specific generalization question explicit: does a common hypothesis satisfy unseen requirements drawn from a distribution over loss functions, example distributions, and thresholds? Its guarantees concern this represented population, without assuming access to a single true reward or resolving whether the population faithfully captures human intentions.

\paragraph{Constrained learning and proxy-boosting.}
PAC constrained learning studies requirements of the form $\E_{x\sim\mathcal D_i}[\ell_i(h,x)]\leq\tau_i$, allowing different losses and distributions across constraints~\cite{chamon2020constrained}.
Thus, an individual Aligned-PAC proxy has a direct predecessor in constrained learning. Our additional emphasis is on a sampled population of requirements and generalization beyond the observed constraints. Oracle reductions for fair classification already combine exponentiated-gradient updates with randomized classifiers~\cite{agarwal2018reductions}.
Proxy-boosting adapts this established approach to thresholded requirements, using a population weighted-learning oracle and fresh sampled losses for weight updates. Its guarantee concerns the observed proxies; it does not establish generalization to unseen proxies.

\paragraph{Alignment verification.}
Value alignment verification formalizes query-efficient evaluation of whether an agent's behavior agrees with human values~\cite{brown2021verification}.
The Seldonian framework provides statistical guarantees against violating specified behavioral constraints, using independent data to assess a selected candidate~\cite{thomas2019preventing}.
Our validation problem introduces uncertainty both across proxies and within each proxy. We study the separate proxy and example budgets required to estimate the gap between training-proxy and population performance, as well as a margin-based certificate for population requirement satisfaction.

\paragraph{Learning across distributions and tasks.}
Collaborative and multi-distribution learning study simultaneous performance across a specified collection of distributions, using boosting-style reweighting~\cite{blum2017collaborative} and game dynamics with adaptive sample allocation~\cite{haghtalab2022ondemand,peng2024multidistribution}.
These are close predecessors to optimization over our training proxies. Generalization to a population of tasks, with sampling at both task and example levels, also appears in Baxter's model of inductive bias learning~\cite{baxter2000bias}.
Aligned-PAC evaluates a common hypothesis against heterogeneous requirements, whereas inductive bias learning uses shared structure to support learning task-specific hypotheses. We leave the general proxy complexity of learning open, building on rather than replacing these existing theories.

\paragraph{Principles and generated evaluations.}
Constitutional AI uses collections of principles to guide AI feedback, while model-written evaluations and automated red teaming use AI systems to construct behavioral evaluations~\cite{bai2022constitutional,perez2022redteaming,perez2023evaluations}.
These provide practical precedents for AI-assisted proxy generation. Our framework asks how a generator defines a distribution of requirements and what independent sampling from that distribution can certify; validity, coverage, and shared blind spots remain separate challenges.

\section{Discussion and Open Problems}
\label{sec:discussion}

Aligned-PAC learning proposes to study alignment through generalization across requirements. Rather than encoding the intended behavior in a single operational criterion, we represent its different aspects by a distribution over proxies, each specifying a loss function, an example distribution, and an acceptable-loss threshold. This separates two questions: whether a model succeeds on the proxies used to train it, and whether that success extends to the broader population of requirements.

The paper develops this perspective in two directions. First, we give rigorous procedures for validating a fixed hypothesis. For the squared-violation objective, a nonadaptive multilevel estimator achieves optimal worst-case sample and proxy complexity for estimating the mis-alignment gap, up to universal constants. A complementary margin-based certificate provides a simpler procedure when all sampled proxies pass. Second, under margin realizability and a population weighted-learning assumption, proxy-boosting converts weighted learning into simultaneous satisfaction of the observed proxies. Fresh sampled losses suffice for its weight updates, and the resulting randomized ensemble satisfies the training proxies in terms of their true expected losses.

\paragraph{The proxy complexity of learning.}
The principal question left open is when success on finitely many training proxies implies success on unseen proxies. More precisely, what assumptions on the hypothesis class, the proxy distribution, and the learning algorithm yield a useful bound on the number of training proxies required to achieve a specified population performance? This is distinct from the proxy complexity of validation: evaluating a hypothesis fixed before sampling does not require controlling how the learning algorithm selects that hypothesis. The proxy-boosting theorem addresses optimization on the observed requirements, but does not establish generalization across requirements.

We leave this question open in its general form. Elementary bounds are possible under restrictive assumptions. For example, for a finite hypothesis class with exact access to each proxy's expected loss, concentration and a union bound give uniform convergence of the bounded outer losses across the class. The challenge is to obtain informative guarantees for expressive models and practical learning algorithms. Even in ordinary learning, available complexity measures often give bounds too loose to explain the observed generalization of modern neural networks. Generalization across proxies introduces an additional question: which properties of the requirements, and of the learner's response to them, determine the effective complexity? A useful theory should also account for finite example budgets, adaptive allocation, and the interaction between sample and proxy complexity. Developing such a theory is beyond the scope of the present paper.

\paragraph{Constructing the proxy distribution.}
All of our guarantees are relative to the chosen proxy distribution. How should this distribution be constructed so that it faithfully represents the intended objective? A large collection of proxies is useful only to the extent that it captures relevant requirements; many near-duplicate proxies may leave important aspects of the task unspecified. The choice of probability mass also determines the relative importance of different requirements. Understanding how proxy diversity, shared structure, and coverage affect alignment is therefore central to applying the framework.

One promising direction is to use an AI agent as a \emph{proxy generator}. Given a description of the intended behavior, such an agent could propose requirements, construct example distributions that exercise them, and suggest loss functions and thresholds. It could also search for missing requirements by examining failures of candidate models. This raises concrete research questions: how to assess the validity and coverage of generated proxies, how to avoid shared blind spots between the generator and the learner, and how to assign weight to the resulting requirements. A fixed generator with independent randomness provides one way to define a proxy distribution for training and validation. Generators that adapt to the learner's behavior may discover more informative requirements, but require an analysis beyond the fixed-distribution sampling assumptions used here.

\paragraph{Implementing the learning oracle.}
The weighted-learning assumption isolates another open problem: obtaining a reliable improvement under changing proxy weights with a practical learning procedure. The sample and computational costs of implementing this oracle, and the effect of occasional failures of its guarantee, are not covered by the proxy-boosting theorem. Studying these questions together with generalization across proxies would connect the optimization framework to end-to-end learning guarantees.

\appendix
\section{Analysis of the Equal-Allocation Estimator}
\label{app:equal-validation}

Condition throughout on the training process, so that $h$ and $\mathcal T_{tr}$ are fixed. For any proxy $t$, set $f_t(u)=[u-\tau_t]_+^2$, let $r_t=\E\ell_t(h,X)$, and let $\bar X_k$ denote the mean of $k$ independent inner losses. The function $f_t$ is convex and has $2$-Lipschitz derivative. Thus, for all $u,v$,
\[
0\leq f_t(u)-f_t(v)-f_t'(v)(u-v)\leq(u-v)^2.
\]
Since $\E\bar X_k=r_t$ and $\operatorname{Var}(\bar X_k)\leq1/(4k)$, taking expectations gives
\begin{equation}\label{eq:squared-bias}
0\leq\E f_t(\bar X_k)-f_t(r_t)\leq\frac1{4k}.
\end{equation}
The same bounds hold after averaging over any proxy distribution $\mathcal Q$.

\begin{proof}[Proof of Theorem~\ref{thm:equal-validation}]
For either $\mathcal Q$, the $N$ empirical proxy scores are independent and lie in $[0,1]$. Hoeffding's inequality and Equation~\ref{eq:squared-bias} imply, with failure probability at most $\delta/2$,
\[
|\widehat R_{N,k}(h;\mathcal Q)-R(h;\mathcal Q)|
\leq\sqrt{\frac{\log(4/\delta)}{2N}}+\frac1{4k}
\leq\frac\epsilon2.
\]
A union bound for the two risks gives gap error at most $\epsilon$. The budgets are $N$ fresh population proxies and $2Nk$ fresh examples.
\end{proof}

For the bias example in the main text, take inner losses distributed as $\operatorname{Bernoulli}(1/2)$ and threshold $1/2$. Symmetry about $1/2$ gives
\[
\E[\bar X_k-1/2]_+^2
=\frac12\E(\bar X_k-1/2)^2=\frac1{8k},
\]
whereas $f_t(r_t)=0$.

\section{Analysis of the Multilevel Estimator}
\label{app:multilevel-validation}

We give a self-contained construction for estimating one risk $R(h;\mathcal Q)$ to accuracy $\eta$, and then apply it to the two risks defining $G(h)$. Condition throughout on the training process, so that $h$ and $\mathcal T_{tr}$ are fixed. Here $\mathcal Q$ is either $\mathcal T$ or $\mathcal T_{tr}$. Recall that all inner losses and thresholds lie in $[0,1]$. For each proxy $t=(\ell_t,\mathcal D_t,\tau_t)$, define
\[
f_t(u)=L(u,\tau_t)=[u-\tau_t]_+^2,
\qquad r_t(h)=\E_{x\sim\mathcal D_t}\ell_t(h,x).
\]
Thus, $R(h;\mathcal Q)=\E_{t\sim\mathcal Q}f_t(r_t(h))$. The function $f_t$ is convex and its derivative is $2$-Lipschitz.

\paragraph{One correction and its independent copies.}
To generate a level-$j$ correction, first draw $t\sim\mathcal Q$, then draw $x_1,\ldots,x_{2^j}$ independently from $\mathcal D_t$ and write $X_b=\ell_t(h,x_b)$. At level zero, set $\Delta_0=f_t(X_1)$. For $j\geq1$, let $q=2^{j-1}$ and set
\[
A_j=\frac1q\sum_{b=1}^q X_b,
\qquad B_j=\frac1q\sum_{b=q+1}^{2q}X_b,
\qquad
\Delta_j=f_t\left(\frac{A_j+B_j}{2}\right)
-\frac{f_t(A_j)+f_t(B_j)}2.
\]
The two halves use the same proxy and are independent conditional on that proxy. The full-sample score uses these same two halves; no additional examples are drawn to compute it.

For chosen integers $J\geq0$ and $N_0,\ldots,N_J\geq1$, repeat this construction $N_j$ times at level $j$. Denote the resulting corrections by $\Delta_j^{(1)},\ldots,\Delta_j^{(N_j)}$. Each copy has its own independently drawn proxy and fresh examples; this independence also holds across levels. A repeated training-proxy index still receives fresh examples. One run of the risk estimator is
\[
Z_{\mathcal Q}=\sum_{j=0}^J\frac1{N_j}\sum_{a=1}^{N_j}\Delta_j^{(a)}.
\]
We specify $J$ and $N_j$ below after bounding its bias and variance.

\paragraph{Telescoping and bias.}
For comparison, let $P_j$ denote the empirical score obtained from a single proxy and $2^j$ examples:
\[
P_j=f_t\left(2^{-j}\sum_{b=1}^{2^j}\ell_t(h,x_b)\right),
\qquad t\sim\mathcal Q,\quad x_b\mid t\overset{\mathrm{i.i.d.}}\sim\mathcal D_t.
\]
Here $P_j$ describes a distribution of scores; it is not an additional term to sample when computing $Z_{\mathcal Q}$. Expectations below average over both the proxy and its examples. In a level-$j$ correction, the full-sample score has the distribution of $P_j$, while each half-sample score has the distribution of $P_{j-1}$. Therefore,
\[
\E\Delta_0=\E P_0,\qquad
\E\Delta_j=\E P_j-\E P_{j-1}\quad(j\geq1),
\]
and the sum telescopes:
\[
\E Z_{\mathcal Q}=\sum_{j=0}^J\E\Delta_j=\E P_J.
\]
The following lemma bounds the bias of an empirical score.

\begin{lemma}[Bias of the empirical squared violation]\label{lem:empirical-squared-bias}
Fix a hypothesis $h$ and a proxy $t$. Let $X_1,\ldots,X_k\in[0,1]$ be independent, identically distributed inner losses conditional on $t$, with conditional mean $r_t(h)$, and let $\bar X_k=k^{-1}\sum_{b=1}^k X_b$. For $f_t(u)=[u-\tau_t]_+^2$,
\[
0\leq\E[f_t(\bar X_k)\mid t]-f_t(r_t(h))
\leq\operatorname{Var}(\bar X_k\mid t)\leq\frac1{4k}.
\]
\end{lemma}
\begin{proof}
Write $\mu=r_t(h)=\E[\bar X_k\mid t]$. Convexity of $f_t$ and Jensen's inequality give the lower bound:
\[
\E[f_t(\bar X_k)\mid t]\geq f_t(\E[\bar X_k\mid t])=f_t(\mu).
\]
For the upper bound, the $2$-Lipschitz derivative of $f_t$ implies the smoothness inequality
\[
f_t(u)\leq f_t(\mu)+f_t'(\mu)(u-\mu)+(u-\mu)^2.
\]
Substituting $u=\bar X_k$ and taking conditional expectations eliminates the linear term, since $\E[\bar X_k-\mu\mid t]=0$. Hence,
\[
\E[f_t(\bar X_k)\mid t]-f_t(\mu)
\leq\E[(\bar X_k-\mu)^2\mid t]
=\operatorname{Var}(\bar X_k\mid t).
\]
Finally, conditional independence and identical distribution imply
\[
\operatorname{Var}(\bar X_k\mid t)
=\frac1{k^2}\sum_{b=1}^k\operatorname{Var}(X_b\mid t)
=\frac{\operatorname{Var}(X_1\mid t)}k.
\]
Because $0\leq X_1\leq1$, we have $X_1^2\leq X_1$ and therefore
\[
\operatorname{Var}(X_1\mid t)
=\E[X_1^2\mid t]-\mu^2
\leq\mu-\mu^2
=\mu(1-\mu)\leq\frac14.
\]
Combining these inequalities proves the claim.
\end{proof}

Applying Lemma~\ref{lem:empirical-squared-bias} with $k=2^J$ and averaging over $t$ yields
\begin{equation}\label{eq:multilevel-bias}
0\leq\E Z_{\mathcal Q}-R(h;\mathcal Q)\leq\frac1{4\cdot2^J}.
\end{equation}

\paragraph{Decay of the correction variance.}
The following second-moment bound also bounds $\operatorname{Var}(\Delta_j)$. It is the reason that fewer copies suffice at finer levels.

\begin{lemma}[Second moments of the corrections]\label{lem:correction-moment}
For every $j\geq0$, $\E\Delta_j^2\leq2^{-2j}$.
\end{lemma}
\begin{proof}
For $j=0$, the correction lies in $[0,1]$. For $j\geq1$, convexity and the $2$-Lipschitz derivative give
\[
|\Delta_j|\leq\frac14(A_j-B_j)^2.
\]
To bound its fourth-moment factor, condition on $t$ and put $q=2^{j-1}$. Pair each loss in the first half with one in the second half: with the notation above, let $W_b=X_b-X_{q+b}$ for $b=1,\ldots,q$. Then
\[
A_j-B_j=\frac1q\sum_{b=1}^q W_b.
\]
Conditional on $t$, the variables $W_b$ are independent, have mean zero, and satisfy $|W_b|\leq1$.

In the expansion of $(\sum_b W_b)^4$, any term containing an index exactly once has conditional expectation zero: independence factors out the mean of the corresponding $W_b$. Thus, only two types of terms remain: one index appearing four times, or two distinct indices each appearing twice. For each unordered pair $b<c$, there are $\binom42=6$ ways to place its two copies of each index among the four factors. Consequently,
\[
\begin{aligned}
\E\left[\left(\sum_{b=1}^q W_b\right)^4\middle|t\right]
&=\sum_{b=1}^q\E[W_b^4\mid t]
+6\sum_{1\leq b<c\leq q}\E[W_b^2\mid t]\E[W_c^2\mid t]\\
&\leq q+6\binom q2
=q+3q(q-1).
\end{aligned}
\]
The inequality uses $\E[W_b^4\mid t]\leq1$ and $\E[W_b^2\mid t]\leq1$, which follow from $|W_b|\leq1$. Dividing by $q^4$ gives
\[
\E[(A_j-B_j)^4\mid t]
\leq\frac{q+3q(q-1)}{q^4}
=\frac{3q^2-2q}{q^4}
\leq\frac3{q^2}.
\]
It follows that $\E[\Delta_j^2\mid t]\leq3/(16q^2)\leq2^{-2j}$. The bound holds uniformly in $t$, so it also holds unconditionally.
\end{proof}

\begin{proof}[Proof of Theorem~\ref{thm:multilevel-validation}]
We first obtain a constant-confidence estimate of each risk, then amplify confidence by taking a median, and finally count the total examples and proxies.

\emph{Allocation and accuracy of one run.} Set $\eta=\epsilon/2$, and define
\[
S_*=\sum_{j=0}^\infty2^{-j/2}=\frac1{1-2^{-1/2}},
\quad J=\left\lceil\log_2(1/\eta)\right\rceil,
\quad N_j=\left\lceil16S_*\eta^{-2}2^{-3j/2}\right\rceil.
\]
For one independent run of $Z_{\mathcal Q}$, Equation~\ref{eq:multilevel-bias} bounds the bias by $\eta/4$. Independence of correction replicates and Lemma~\ref{lem:correction-moment} give
\[
\operatorname{Var}(Z_{\mathcal Q})
\leq\sum_{j=0}^J\frac{2^{-2j}}{N_j}
\leq\frac{\eta^2}{16S_*}\sum_{j=0}^J2^{-j/2}
\leq\frac{\eta^2}{16}.
\]
Chebyshev's inequality therefore implies
\[
\mathbb P\bigl(|Z_{\mathcal Q}-R(h;\mathcal Q)|>\eta\bigr)
\leq\frac{\eta^2/16}{(3\eta/4)^2}=\frac19.
\]
\emph{Confidence amplification and the gap estimate.} Repeat the entire construction independently $B$ times and take the median, where $B$ is the smallest odd integer at least $8\log(2/\delta)$. Since a run fails with probability at most $1/4$, Hoeffding's inequality applied to the failure indicators bounds the probability of a majority of failures by $\exp(-B/8)\leq\delta/2$. Thus, the median estimates this risk to accuracy $\eta$ with failure probability at most $\delta/2$. Apply this procedure separately to the two risks and subtract their medians. A union bound establishes Equation~\ref{eq:gap-accuracy}.

\emph{Sample and proxy budgets.} One run uses exactly $\sum_{j=0}^J N_j2^j$ examples. Using $\lceil a\rceil\leq a+1$ and $2^J<2/\eta$,
\[
\sum_{j=0}^J N_j2^j
\leq16S_*\eta^{-2}\sum_{j=0}^J2^{-j/2}
+\sum_{j=0}^J2^j
\leq16S_*^2\eta^{-2}+\frac4\eta.
\]
It uses $\sum_j N_j\leq\sum_j N_j2^j=O(\eta^{-2})$ proxy draws. For $\mathcal Q=\mathcal T$, these are fresh population proxies; for $\mathcal Q=\mathcal T_{tr}$, they are draws of existing training-proxy indices. Repetition and estimating both risks therefore give the asserted budgets. Every allocation is predetermined, so these bounds hold deterministically, not merely in expectation.
\end{proof}

\section{Lower Bounds for Validation}
\label{app:validation-lower}

We use an estimation-to-testing reduction. Suppose two instances have gaps $G_0,G_1$ separated by more than $2\epsilon$. An estimator satisfying Equation~\ref{eq:gap-accuracy} on both instances induces a binary decision rule with error at most $\delta$ under either instance, by comparing its output to the midpoint of $G_0,G_1$. Let $P_0,P_1$ denote the distributions of the entire evaluation transcript. Data processing for relative entropy gives
\begin{equation}\label{eq:testing-kl}
D_{\mathrm{KL}}(P_0\Vert P_1)
\geq\operatorname{kl}(1-\delta,\delta)
=(1-2\delta)\log\frac{1-\delta}{\delta}
\geq\frac14\log(1/\delta)
\end{equation}
for $0<\delta\leq1/4$, where $\operatorname{kl}$ denotes binary relative entropy. We also use
\[
D_{\mathrm{KL}}(\operatorname{Bern}(p_0)\Vert\operatorname{Bern}(p_1))
\leq\frac{(p_0-p_1)^2}{p_1(1-p_1)}.
\]
Proxy distributions are available through the sampling interfaces specified in the framework; the unknown parameters are not revealed as part of their descriptions.

\begin{proof}[Proof of the sample lower bound in Theorem~\ref{thm:validation-lower}]
Take a singleton hypothesis class and two proxy types. Proxy $t_0$ has loss identically zero and threshold zero. Proxy $t_1$ has threshold zero and loss $X$, where $X\sim\operatorname{Bern}(p)$. Let $\mathcal T$ put probability $1/2$ on each type. Condition on a training outcome consisting only of $t_0$ proxies; this has positive probability, and its observations contain no information about $p$. Then
\[
R(h;\mathcal T_{tr})=0,\qquad G(h)=\tfrac12p^2.
\]
Take $p_0=1/2$, $p_1=1/2+6\epsilon$, and $0<\epsilon\leq1/24$. The gap separation is $3\epsilon+18\epsilon^2>2\epsilon$. Fresh proxy labels have the same distribution in both instances and provide no information about $p$; only examples from $t_1$ are informative. Each such example contributes at most $192\epsilon^2$ to the relative entropy. Even an evaluator allowed to choose proxy types directly and allocate samples adaptively, subject to a deterministic budget $M$, has transcript relative entropy at most $192M\epsilon^2$, by the chain rule. Equation~\ref{eq:testing-kl} therefore implies
\[
M\geq\frac{\log(1/\delta)}{768\epsilon^2}.
\]
This lower bound for the stronger access model also applies to the stated evaluation protocol.
\end{proof}

\begin{proof}[Proof of the proxy lower bound in Theorem~\ref{thm:validation-lower}]
Again use a singleton hypothesis class and two proxy types, now with deterministic inner losses zero and one, respectively, and thresholds zero. Let their probabilities under $\mathcal T$ be $1-p$ and $p$. Condition on a training outcome consisting only of the zero-loss proxy. The gap is $G(h)=p$. We may reveal the exact score of every proxy for free: neither additional examples nor repeated evaluation of an existing proxy supplies information about $p$.

For $p_0=1/2$ and $p_1=1/2+3\epsilon$, the gap separation is $3\epsilon>2\epsilon$. For $\epsilon\leq1/24$, each fresh proxy draw contributes at most $48\epsilon^2$ to the relative entropy. With a deterministic budget of $N$ fresh proxies, Equation~\ref{eq:testing-kl} gives
\[
N\geq\frac{\log(1/\delta)}{192\epsilon^2}.
\]
The chain-rule argument also permits adaptive stopping within that budget. Both constructions concern the conditional evaluation problem after training and establish the claimed worst-case lower bounds.
\end{proof}

\section{Proof of the Margin-Based Validation Certificate}
\label{app:margin-validation}

\begin{proof}[Proof of Theorem~\ref{thm:margin-validation}]
Fix $h$ and $\mathcal Q$. For a proxy $t$ satisfying $r_t(h)>\tau_t+\gamma$, the one-sided Hoeffding inequality gives
\[
\mathbb P\bigl(\widehat r_t(h)\leq\tau_t\mid t\bigr)
\leq\mathbb P\bigl(\widehat r_t(h)-r_t(h)\leq-\gamma\mid t\bigr)
\leq e^{-2k_v\gamma^2}.
\]
Thus, for one randomly drawn proxy and its examples,
\[
\begin{aligned}
\mathbb P\bigl(\widehat r_t(h)>\tau_t\bigr)
&\geq R_\gamma(h;\mathcal Q)\bigl(1-e^{-2k_v\gamma^2}\bigr).
\end{aligned}
\]
The $n_v$ proxy evaluations are independent, so whenever $R_\gamma(h;\mathcal Q)>\epsilon$,
\[
\begin{aligned}
\mathbb P(\text{all proxies pass})
&\leq\left[1-\epsilon\bigl(1-e^{-2k_v\gamma^2}\bigr)\right]^{n_v}\\
&\leq\exp\!\left[-n_v\epsilon\bigl(1-e^{-2k_v\gamma^2}\bigr)\right].
\end{aligned}
\]
The stated $k_v$ ensures $e^{-2k_v\gamma^2}\leq1/2$, and the stated $n_v$ makes the last expression at most $\delta$. If $R_\gamma(h;\mathcal Q)\leq\epsilon$, the claimed disjunction holds automatically. The budget bounds follow by multiplying $n_v$ by $k_v$.
\end{proof}

\section{Analysis of Proxy-Boosting}
\label{app:proxy-boosting}

\begin{proof}[Proof of Theorem~\ref{thm:proxy-boosting}]
Condition on a training-proxy sequence for which $r_i(h^*)\leq\tau_i-\gamma$ for all $i$. Such sequences have probability one. At round $s$, let $\mathcal F_s$ include the entire history immediately before the fresh weight-update examples are drawn, including the oracle output $h_s$ and weights $q^{(s)}$. Define
\[
v_i^{(s)}=r_i(h_s)-\tau_i,
\qquad e_i^{(s)}=\widehat r_i^{(s)}-r_i(h_s).
\]
The oracle assumption implies, for every round,
\begin{equation}\label{eq:population-proxy-edge-proof}
\sum_iq_i^{(s)}v_i^{(s)}
\leq\sum_iq_i^{(s)}(r_i(h^*)-\tau_i)+\xi
\leq-\gamma+\xi=-\rho.
\end{equation}

First apply the multiplicative-weights argument to the observed violations. Set $w_i^{(1)}=1$, $w_i^{(s+1)}=w_i^{(s)}e^{\eta\widehat v_i^{(s)}}$, and $W_s=\sum_iw_i^{(s)}$. Since $\widehat v_i^{(s)}\in[-1,1]$, Hoeffding's lemma gives the pathwise bound
\[
\log\frac{W_{s+1}}{W_s}
\leq\eta\sum_iq_i^{(s)}\widehat v_i^{(s)}+\frac{\eta^2}{2}.
\]
Summing over rounds, using $W_1=n$, and lower-bounding $W_{B+1}$ by each individual weight yields
\begin{equation}\label{eq:sampled-proxy-regret}
\max_i\frac1B\sum_{s=1}^B\widehat v_i^{(s)}
\leq\frac1B\sum_{s=1}^B\sum_iq_i^{(s)}\widehat v_i^{(s)}
+\frac{\log n}{\eta B}+\frac\eta2.
\end{equation}
This inequality holds for every realization of the data.

Next control the difference between observed and population losses. Conditional on $\mathcal F_s$, each $e_i^{(s)}$ is the average of $k$ independent centered variables with range length one. Therefore, for every real $\lambda$,
\[
\E\left[e^{\lambda e_i^{(s)}}\mid\mathcal F_s\right]
\leq\exp\left(\frac{\lambda^2}{8k}\right).
\]
For the weighted error $z_s=\sum_iq_i^{(s)}e_i^{(s)}$, conditional independence across proxies similarly gives
\[
\E\left[e^{\lambda z_s}\mid\mathcal F_s\right]
\leq\exp\left(\frac{\lambda^2}{8k}\sum_i(q_i^{(s)})^2\right)
\leq\exp\left(\frac{\lambda^2}{8k}\right).
\]
Iterating these conditional moment-generating-function bounds over rounds and applying the Chernoff bound shows, for any $a>0$,
\[
\mathbb P\left(-\frac1B\sum_s e_i^{(s)}>a\right)
\leq e^{-2Bka^2},
\qquad
\mathbb P\left(\frac1B\sum_s z_s>a\right)
\leq e^{-2Bka^2}.
\]
Set $a=\sqrt{\log((n+1)/\delta)/(2Bk)}$. A union bound over the $n$ individual inequalities and the one weighted inequality gives, with probability at least $1-\delta$,
\[
\frac1B\sum_s v_i^{(s)}
\leq\frac1B\sum_s\widehat v_i^{(s)}+a\quad\text{for every }i,
\qquad
\frac1B\sum_s\sum_iq_i^{(s)}\widehat v_i^{(s)}
\leq\frac1B\sum_s\sum_iq_i^{(s)}v_i^{(s)}+a.
\]
Combine these statements with Equations~\ref{eq:population-proxy-edge-proof} and~\ref{eq:sampled-proxy-regret}. Since $B^{-1}\sum_s v_i^{(s)}=r_i(\pi)-\tau_i$, the result is Equation~\ref{eq:sampled-proxy-bound}.

For the stated parameter choice, $\eta/2=\rho/4$, $\log n/(\eta B)\leq\rho/4$, and $2a\leq\rho/4$, leaving a margin of at least $\rho/4$. The argument allows all oracle choices and weights to depend on the preceding history; it only requires the oracle guarantee at each call and fresh subsequent weight-update samples. Finally, averaging over the training-proxy sequence preserves the probability bound.
\end{proof}

\end{document}